\documentclass[ 10 pt, conference]{ieeeconf}  

\IEEEoverridecommandlockouts                              
\usepackage{float}
\usepackage{hyperref}

\usepackage{amsfonts}
\usepackage{amsmath}
\usepackage{amssymb}

\usepackage{amsthm}

\usepackage{mathtools}
\usepackage{xcolor}

\let\parens\undefined
\newcommand{\parens}[1]{{\left(#1\right)}}

\let\braces\undefined
\newcommand{\braces}[1]{{\left\{#1\right\}}}

\let\abs\undefined
\newcommand{\abs}[1]{{\lvert#1\rvert}}

\let\norm\undefined
\newcommand{\norm}[1]{{\lVert#1\rVert}}

\let\mrm\undefined
\newcommand{\mrm}[1]{{\mathrm{#1}}}

\let\mbb\undefined
\newcommand{\mbb}[1]{{\mathbb{#1}}}

\let\mbf\undefined
\newcommand{\mbf}[1]{{\mathbf{#1}}}

\let\mcal\undefined
\newcommand{\mcal}[1]{{\mathcal{#1}}}

\theoremstyle{plain}
\newtheorem{theorem}{Theorem}
\newtheorem{lemma}{Lemma}

\title{\LARGE \bf
Chordal Sparsity for Lipschitz Constant Estimation \\of Deep Neural Networks
}

\author{Anton Xue, Lars Lindemann, Alexander Robey, Hamed Hassani, George J. Pappas, and Rajeev Alur ${}^\dagger$
    \thanks{${}^\dagger$ The authors are with the School of Engineering and Applied Sciences, University of Pennsylvania, Philadelphia, PA, USA.
    {\tt\small \{antonxue,larsl,arobey1,hassani,pappasg,alur\}@ seas.upenn.edu}}
}

\begin{document}

\maketitle
\thispagestyle{empty}
\pagestyle{empty}

\begin{abstract}
Lipschitz constants of neural networks allow for guarantees of robustness in image classification, safety in controller design, and generalizability beyond the training data.
As calculating Lipschitz constants is NP-hard, techniques for estimating Lipschitz constants must navigate the trade-off between scalability and accuracy.
In this work, we significantly push the scalability frontier of a semidefinite programming technique known as LipSDP while achieving zero accuracy loss.
We first show that LipSDP has \emph{chordal sparsity}, which allows us to derive a chordally sparse formulation that we call Chordal-LipSDP.
The key benefit is that the main computational bottleneck of LipSDP, a large semidefinite constraint, is now decomposed into an \emph{equivalent} collection of smaller ones --- allowing Chordal-LipSDP to outperform LipSDP particularly as the network depth grows.
Moreover, our formulation uses a tunable sparsity parameter that enables one to gain tighter estimates without incurring a significant computational cost.
We illustrate the scalability of our approach through extensive numerical experiments.
\end{abstract}

\section{Introduction}
\label{sec:introduction}

Neural networks are arguably the most common choice of function approximators used in machine learning and artificial intelligence. Their success is well documented in the literature and showcased in various applications, e.g., in solving the game Go  \cite{silver2018general} and in handwritten character recognition \cite{lecun1998gradient}. However, neural networks have shown to be non-robust, i.e., their outputs may be sensitive to small changes in the inputs and result in large deviations in the outputs \cite{goodfellow2014explaining, su2019one}.
Hence, it is often unclear what a neural network exactly learns and how it can generalize to previously unseen data. This is a particular concern in safety-critical applications like perception in autonomous driving, where one would like to be robust and even obtain a robustness certificate. A way to measure the robustness of a neural network $f:\mathbb{R}^{n_1}\to \mathbb{R}^m$ is to calculate its Lipschitz constant $L$ that satisfies
\begin{align*}
    \|f(x)-f(y)\|\le L \|x-y\| \quad \text{for all \(x,y\in\mathbb{R}^{n_1}\)}.
\end{align*}

The exact calculation of $L$ is NP-hard and hence poses computational challenges \cite{virmaux2018lipschitz,jordan2020exactly}.
Therefore, past efforts have focused on estimating upper bounds on the Lipschitz constant $L$ in computationally efficient ways.
A key difficulty here is to appropriately model the nonlinear activation functions within a neural network.
For feedforward neural networks, the authors in~\cite{fazlyab2019efficient} abstract activation functions using incremental quadratic constraints~\cite{accikmecse2011observers}.
These are then formed into a convex semidefinite program, referred to as LipSDP, whose solution yields tight upper bounds on \(L\).
As the size of the neural network grows, however, the general formulation of LipSDP becomes computationally intractable.
One may partially alleviate this issue by selectively reducing the number of optimization variables, which will induce sparsity into LipSDP at the cost of a looser bound.
Still, this does not address the core computational bottleneck of LipSDP, which is that the solver must process a large semidefinite constraint whose dimension scales with the number of neurons.

In this paper, we study \emph{computationally efficient formulations of LipSDP}.
In particular, we introduce a variant of LipSDP that exhibits \emph{chordal sparsity}~\cite{vandenberghe2015chordal, zheng2019chordal}, which allows us to decompose a large semidefinite constraint into an equivalent collection of smaller ones.
Moreover, our formulation has a \emph{tunable sparsity parameter}, enabling one to trade-off between efficiency and accuracy.
We call our decomposed semidefinite program \emph{Chordal-LipSDP}, and study its theoretical properties and computational performance in this paper.
The contributions of our work are as follows:
\begin{itemize}
    \item We introduce a variant of LipSDP formulated in terms of a sparsity parameter \(\tau\) and precisely characterize its chordal sparsity pattern.
    This allows us to decompose LipSDP, which is a large semidefinite constraint, into a collection of smaller ones, yielding an equivalent problem that we call Chordal-LipSDP.
    
    \item We present numerical evaluations and observe that Chordal-LipSDP is significantly faster than LipSDP, especially for deeper networks, without accuracy loss relative to LipSDP.
    Furthermore, adjusting \(\tau\)  allows Chordal-LipSDP to obtain rapidly tightening bounds on \(L\) without incurring a high performance penalty.
    
    \item We make an open-source implementation available at \href{https://github.com/AntonXue/chordal-lipsdp}{{\tt \small github.com/AntonXue/chordal-lipsdp}}.

\end{itemize}

\subsection{Related Work}

There has been a great interest in the machine learning and control communities towards efficiently and accurately estimating Lipschitz constants of neural networks. Indeed, it has been shown that there is a close connection between the Lipschitz constant of a neural network and its ability to generalize \cite{bartlett2017spectrally}.
The authors in~\cite{miyato2018spectral} were among the first to normalize weights of a neural network based on the Lipschitz constant.
In control, Lipschitz constants of neural network-based control laws can be used to obtain stability or safety guarantees \cite{jin2020stability,lindemann2021learning}. Training neural networks with a desired Lipschitz constant is, however, difficult. In practice, one has to either solve constrained optimization problems, e.g., \cite{gouk2021regularisation}, or iteratively bootstrap training parameters. As a consequence, one is interested in obtaining Lipschitz certificates of neural networks. In \cite{virmaux2018lipschitz,jordan2020exactly}, it was shown that the exact calculation of the Lipschitz constant is NP-hard. As estimating Lipschitz constants is computationally challenging, we are here motivated to efficiently estimate the Lipschitz constants of neural networks. 

Broadly, there are two ways for estimating Lipschitz constants of general nonlinear functions, either sampling-based as in \cite{wood1996estimation} and \cite{chakrabarty2020safe}, or using optimization techniques~\cite{fazlyab2019efficient, latorre2020lipschitz}.
A naive approach is to calculate the product of the norm of the weights of each individual layer.
The authors in \cite{virmaux2018lipschitz} follow a similar idea and obtain tighter Lipschitz constants using singular value decomposition and maximization over the unit cube.
This, however, still becomes quickly computationally intractable for large neural networks.
Tighter bounds have been obtained in \cite{combettes2020lipschitz}, capturing cross-layer dependencies using compositions of nonexpansive averaged operators.
Again, however, this approach does not scale well with the number of layers.
While these works estimate global Lipschitz constants, it was shown in \cite{avant2021analytical} that estimating local Lipschitz can be done more efficiently.

In this paper, we build on the LipSDP framework presented in \cite{fazlyab2019efficient}, which amounts to solving a semidefinite optimization program.
LipSDP abstracts activation functions into quadratic constraints and allows one to encode rich layer-to-layer relations, allowing for trade-offs in accuracy and efficiency.
While LipSDP considers the $l_2$-norm, general $l_p$-norms on the input-output relation of a neural network can be conservatively obtained using the equivalence of norms.
The authors in \cite{latorre2020lipschitz} present LiPopt, which is a polynomial optimization framework that allows one to calculate tight estimates of Lipschitz constants for $l_2$ and $l_\infty$-norms.
For $l_2$-norms, however, LipSDP is empirically shown to have tighter bounds.
The exact computation of the Lipschitz constant under $l_1$ and $l_\infty$ norms was presented in \cite{jordan2020exactly} by solving a mixed integer linear program.
Lipschitz continuity of a neural network with respect to its training parameters has been analyzed in \cite{herrera2020estimating}.

We show that a particular formulation of LipSDP satisfies \emph{chordal sparsity}~\cite{vandenberghe2015chordal, zheng2019chordal}, from which we apply chordal decomposition to obtain Chordal-LipSDP.
Applications of chordal sparsity have also been explored in other domains \cite{mason2014chordal,ihlenfeld2022faster,chen2020chordal,newton2021exploiting}. The key benefit of exploiting chordal sparsity is that a large semidefinite constraint is decomposed into an equivalent collection of smaller ones, in particular, allowing us to scale to deeper networks.
This equivalence also means that LipSDP and Chordal-LipSDP will compute \emph{identical} estimates of the Lipschitz constant.


\section{Background and Problem Formulation}
\label{sec:background}

In this section, we state the problem formulation and provide background on LipSDP and chordal sparsity.

\subsection{Lipschitz Constant Estimation of Neural Networks}

We consider feedforward neural networks $f:\mathbb{R}^{n_1}\to \mathbb{R}^m$ with $K \ge 2$ layers, i.e., $K - 1$ hidden layers and one linear output layer. From now on, let \(x_1 \in \mbb{R}^{n_1}\) denote the input of the neural network. The output of the neural network is recursively computed for layers \(k = 1, \ldots, K - 1\) as
\begin{align}
\label{eq:ffnet}
    f(x_1) \coloneqq W_K x_K + b_K,
    \quad
     x_{k+1} \coloneqq \phi (W_k x_k + b_k),
\end{align}
  where $W_k$ and $b_k$ are the weight matrices and bias vectors of the $k$th layer, respectively, that are assumed to be of appropriate size. We denote the dimensions of $x_2,\hdots,x_K$ by $n_2,\hdots,n_K\in \mathbb{N}$. The function \(\phi(u)\coloneqq \mrm{vcat}(\varphi(u_1), \varphi(u_2) \hdots )\) is the stack vector of activation functions $\varphi$, e.g., ReLU or tanh activation functions, that are applied element-wise. Throughout the paper, we assume that the same activation function is used across all layers.


\subsection{LipSDP}
We now present LipSDP~\cite{fazlyab2019efficient} in a way that enables us later to conveniently characterize the chordal sparsity pattern of LipSDP.
First, let \(\mbf{x} \coloneqq \mrm{vcat}(x_1, \ldots, x_K) \in \mbb{R}^{N}\) be a stack of the state vectors with \(N \coloneqq \sum_{k = 1}^{K} n_k\).
By defining
{\small
\begin{align*}
\begin{split}
    A   
    &\coloneqq \begin{bmatrix}
        W_1 & \cdots & 0 & 0 \\
        \vdots & \ddots & \vdots & \vdots \\
        0 & \cdots & W_{K-1} & 0
    \end{bmatrix},
    \enskip
    B
    \coloneqq \begin{bmatrix}
        0 & I_{n_2} & \cdots & 0 \\
        \vdots & \vdots & \ddots & \vdots \\
        0 & 0 & \cdots & I_{n_K}
    \end{bmatrix}
\end{split}
\end{align*}
}
\noindent and \(b \coloneqq \mrm{vcat}(b_1, b_2, \ldots, b_{K-1})\)
we can rewrite the dynamics of~\eqref{eq:ffnet} as \(B \mbf{x} = \phi(A \mbf{x} + b)\),
where $\phi : \mbb{R}^{N_f} \to \mbb{R}^{N_f}$ is a \(N_f\)-height stack of \(\varphi\) with $N_f \coloneqq n_2 + \cdots + n_K$. To deal with the nonlinear activation function \(\phi\) efficiently, the key idea in LipSDP is to abstract \(\phi\) using incremental quadratic constraints~\cite{accikmecse2011observers}. 
In particular, LipSDP considers a family of symmetric indefinite matrices \(\mcal{Q}\)  such that any matrix \(Q \in \mcal{Q}\) satisfies
\begin{align}
\label{eq:sector-bounded-Q}
    \begin{bmatrix} u - v \\ \phi(u) - \phi(v) \end{bmatrix}^\top
    Q
    \begin{bmatrix} u - v \\ \phi(u) - \phi(v) \end{bmatrix}
    \geq 0
\end{align}
for all \(u, v \in \mbb{R}^{N_f}\).
In the case where each element \(\varphi\) of \(\phi\) is \([\underline{s}, \overline{s}]\)-\emph{sector-bounded}, i.e., where its subgradients satisfy \(\partial \varphi \subseteq [\underline{s}, \overline{s}]\), then one possible parameterization of \(\mcal{Q}\) is 
\begin{align*}
    \mcal{Q} \coloneqq
    \braces{
        \begin{bmatrix} A \\ B \end{bmatrix}^\top
        \begin{bmatrix}
            -2 \underline{s} \overline{s} T & (\underline{s} + \overline{s}) T \\
            (\underline{s} + \overline{s}) T & - 2 T
        \end{bmatrix}
        \begin{bmatrix} A \\ B \end{bmatrix}
    : \gamma_\alpha \geq 0
    }
\end{align*}
where $T$ is a dense matrix that is parametrized by \(\gamma_{\alpha}\). In this paper, we fix an integer \(\tau \geq 0\) and define $T$ as follows\footnote{We specialize \(T\) to be \(\tau\)-banded whereas LipSDP permits \(T\) to be dense. However this restriction induces a chordally sparse structure in LipSDP.}
\begin{align*}
\begin{split}
    T &\coloneqq \sum_{i = 1}^{N_f} (\gamma_\alpha)_{ii} e_i e_i ^\top
        + \sum_{(i, j) \in \mcal{I}_\tau}
            (\gamma_\alpha)_{ij} (e_i - e_j) (e_i - e_j)^\top, \\
    I_\tau &\coloneqq \{(i,j) = 1 \leq i < j \leq N_f, \enskip j - i \leq \tau\}.
\end{split}
\end{align*}

By tuning the value of \(\tau\), we obtain different formulations of \(\mcal{Q}\) that all provide over-approximations of \(\phi\) as in~\eqref{eq:sector-bounded-Q} while allowing us to trade-off on the spectrum of sparsity and accuracy.
In the sparsest case, i.e., \(\tau = 0\), the matrix \(T\) is a nonnegative diagonal matrix and \(\gamma_\alpha \in \mbb{R}_+ ^{N_f}\), while in the densest case, i.e., \(\tau = N_f - 1\), the matrix \(T\) is fully dense and parameterized by \(\gamma_{\alpha} \in \mbb{R}_+ ^{1 + \cdots + N_f}\).

To formulate our variant of  LipSDP, we define the linearly-parametrized matrix-valued functions
\begin{align*}
    Z_\alpha (\gamma_{\alpha})
    &\coloneqq \begin{bmatrix} A \\ B \end{bmatrix}^\top
        \begin{bmatrix}
            -2 \underline{s} \overline{s} T & (\underline{s} + \overline{s}) T \\
            (\underline{s} + \overline{s}) T & - 2 T
        \end{bmatrix}
    \begin{bmatrix} A \\ B \end{bmatrix}, \\
    Z_\ell (\gamma_{\ell})
        &\coloneqq E_K^\top (W_K ^\top W_K) E_K
          - \gamma_{\ell} E_1 ^\top E_1, \\
    E_k &\coloneqq
        \begin{bmatrix} \cdots & 0 & I_{n_k} & 0 & \cdots \end{bmatrix}
        \in \mbb{R}^{n_k \times N},
\end{align*}
where \(T\) is defined as above, \(\gamma_\ell \in \mbb{R}_+\), and \(E_k\) is the \(k\)th block-index selector such that \(x_k = E_k \mbf{x}\).
Now combine the above terms as
\begin{align}\label{eq:Z_gamma}
    Z(\gamma) \coloneqq
    Z_\alpha (\gamma_\alpha) + Z_\ell (\gamma_\ell)
    \in \mbb{S}^{N},
\end{align}
then LipSDP is the following semidefinite program:
\begin{align}
\label{eq:p1}
    \underset{\gamma \geq 0}{\text{minimize}} \enskip \gamma_{\ell}
    \quad
    \text{subject to} \enskip Z(\gamma) \preceq 0
\end{align}
If \(\gamma_\ell ^\star\) is the optimal value of~\eqref{eq:p1}, then the Lipschitz constant of \(f\) is upper-bounded by \((\gamma_{\ell} ^\star)^{1/2}\), see \cite{fazlyab2019efficient}.
\footnote{
    After the publication of this work, we were made aware of~\cite{pauli2021training}, which shows a counterexample to a key claim originally made in LipSDP~\cite{fazlyab2019efficient} (and since amended).
    In general, only the (\(\tau = 0\)) case will certify the Lipschitz constant.
    However, our techniques for analyzing the relevant sparsity structures remain valid.
}
That is,
\begin{align*}
    \norm{f(x) - f(y)} \leq (\gamma_\ell ^\star)^{1/2} \norm{x - y}
    \quad \text{for all \(x, y \in \mbb{R}^{n_1}\)}.
\end{align*}


\subsection{Chordal Sparsity}

Chordal sparsity connects graph theory and sparse matrix decomposition~\cite{griewank1984existence, vandenberghe2015chordal}.
In the context of this paper, we aim to solve the potentially large-scale semidefinite program~\eqref{eq:p1} using chordal sparsity in \(Z(\gamma)\).
This is done by decomposing the semidefinite constraint \(Z(\gamma) \preceq 0\) into an equivalent collection of smaller \(Z_k \preceq 0\) constraints, which we demonstrate in Section~\ref{sec:chordal-lipsdp}.

\subsubsection{Chordal Graphs and Sparse Matrices}
A graph \(\mcal{G}(\mcal{V}, \mcal{E})\) consists of vertices \(\mcal{V} \coloneqq \{1, \ldots, n\}\) and edges \(\mcal{E} \subseteq \mcal{V} \times \mcal{V}\).
We assume that \(\mcal{E}\) is symmetric, i.e. \((i,j) \in \mcal{E}\) implies \((j,i) \in \mcal{E}\), and so \(\mcal{G}(\mcal{V}, \mcal{E})\) is an undirected graph.
We say that the vertices \(\mcal{C} \subseteq \mcal{V}\) form a \emph{clique} if \(u, v \in \mcal{C}\) implies \((u, v) \in \mcal{E}\), and let \(\mcal{C}(i)\) be the \(i\)th vertex of \(\mcal{C}\) under the natural ordering.
A \emph{maximal clique} is a clique that is not strictly contained within another clique.
A \emph{cycle} of length \(l\) is a sequence of vertices \(v_1, \ldots, v_l\) with \((v_l, v_1) \in \mcal{E}\) and adjacent connections \((v_{i}, v_{i+1}) \in \mcal{E}\).
A \emph{chord} is any edge that connects two nonadjacent vertices in a cycle, and we say that a graph is \emph{chordal} if every cycle of length four has at least one chord~\cite{vandenberghe2015chordal}.

An edge set \(\mcal{E}\) can dually describe the sparsity pattern of a matrix.
Given a graph \(\mcal{G}(\mcal{V}, \mcal{E})\), define the set of symmetric matrices of size \(n\) with sparsity pattern \(\mcal{E}\) as
\begin{align}
\label{eq:SE}
    \mbb{S}^n (\mcal{E}) \coloneqq
    \{X \in \mbb{S}^{n} : X_{ij} = X_{ji} = 0 \enskip \text{if} \enskip (i,j) \not\in \mcal{E}\}.
\end{align}
If in addition \(\mcal{G}(\mcal{V}, \mcal{E})\) is chordal and \(X \in \mbb{S}^{n} (\mcal{E})\), then we say that \(X\) has \emph{chordal sparsity} or is \emph{chordally sparse}.
For \(X\) with sparsity \(\mcal{E}\), we say that \(X_{ij}\) is \emph{dense} if \((i,j) \in \mcal{E}\), and that it is \emph{sparse} otherwise.

\subsubsection{Chordal Decomposition of Sparse Matrices}
For a chordally sparse \(X \in \mbb{S}^n (\mcal{E})\), useful decompositions can be analyzed through the cliques of \(\mcal{G}(\mcal{V}, \mcal{E})\).
Given a clique \(\mcal{C}_k \subseteq \mcal{V}\), define its block-index matrix as follows:
\begin{align*}
    (E_{\mcal{C}_k})_{ij}
        = \text{\(1\) if \(\mcal{C}_k(i) = j\) else \(0\)},
        \quad
    E_{\mcal{C}_k} \in \mbb{R}^{\abs{\mcal{C}_k} \times n}.
\end{align*}
By decomposing a chordally sparse matrix with respect to its maximal cliques, a key result in sparse matrix analysis allows us to deduce the semidefiniteness of a large matrix with respect to a collection of smaller matrices.
\begin{lemma}[Theorem 2.10~\cite{zheng2019chordal}]
\label{lem:chordal-psd}
    Let \(\mcal{G}(\mcal{V}, \mcal{E})\) be a chordal graph and let \(\{\mcal{C}_1, \ldots, \mcal{C}_p\}\) be the set of its maximal cliques.
    Then \(X \in \mbb{S}^n (\mcal{E})\) and \(X \succeq 0\) if and only if there exists \(X_k \in \mbb{S}^{\abs{\mcal{C}_k}}\) such that each \(X_k \succeq 0\) and
    \begin{align}
        \label{eq:chordal-decomp}
        X = \sum_{k = 1}^{p} E_{\mcal{C}_k}^\top X_k E_{\mcal{C}_k}.
    \end{align}
\end{lemma}
We say that \eqref{eq:chordal-decomp} is a \emph{chordal decomposition} of \(X\) by \(\mcal{C}_1, \ldots, \mcal{C}_p\), and such a decomposition allows us to solve a large semidefinite constraint using an equivalent collection of smaller ones.

\section{Chordal Decomposition of LipSDP}
\label{sec:chordal-lipsdp}

In this section, we present Chordal-LipSDP, which is a chordally sparse formulation of LipSDP.
We first identify the sparsity pattern for \(Z(\gamma)\) in Theorem~\ref{thm:Z-sparsity} and then present Chordal-LipSDP in Theorem~\ref{thm:max-cliques-Z} as a chordal decomposition of LipSDP.
An equivalence result is then stated in Theorem~\ref{thm:equiv-probs}. The proofs of our results can be found in the appendix.

Our goal is to construct the edge set $\mcal{E}$ of a chordal graph \(\mcal{G}(\mcal{V}, \mcal{E})\) with vertices \(\mcal{V} \coloneqq \{1, \ldots, N\}\) such that \(Z(\gamma) \in \mbb{S}^{N} (\mcal{E})\).
To gain intuition for \(\mcal{E}\), we plot the dense entries of \(Z(\gamma)\) in Figure~\ref{fig:Ztaus} where the \((i,j)\) square is dark if \((Z(\gamma))_{ij}\) is dense, i.e., \((i,j) \in \mcal{E}\).

\begin{figure}[h]
\centering
\begin{minipage}{0.16\textwidth}
    \includegraphics[width=1.0\textwidth]{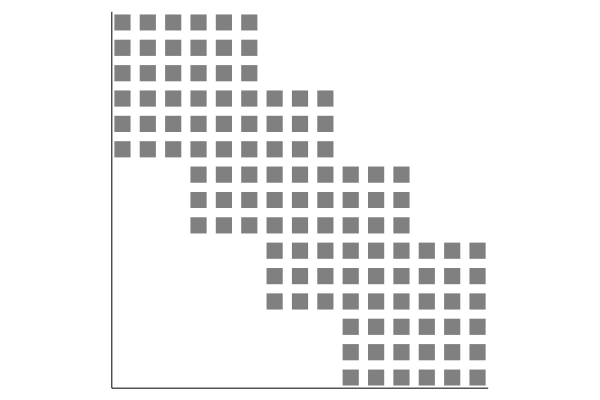}
\end{minipage}%
\begin{minipage}{0.16\textwidth}
    \includegraphics[width=1.0\textwidth]{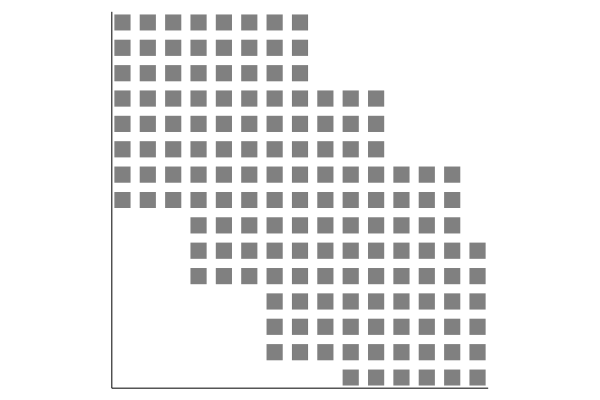}
\end{minipage}%
\begin{minipage}{0.16\textwidth}
    \includegraphics[width=1.0\textwidth]{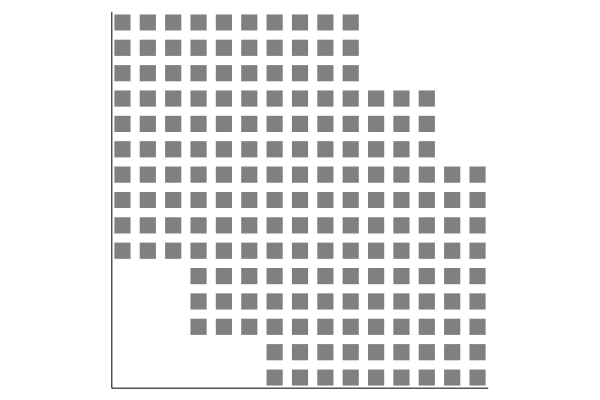}
\end{minipage}

\caption{The sparsity of \(Z(\gamma)\) for \(\tau = 0, 2, 4\) with dimensions \((3,3,3,3,3)\).
For each increment of \(\tau\), each block grows by one unit on the bottom and right, and corresponds to a maximal clique of \(\mcal{G}(\mcal{V}, \mcal{E})\).
As \(\tau\) increases the number of blocks (maximal cliques) will decrease as the lower-right blocks become overshadowed.
At \(\tau = 0\) we have what~\cite{fazlyab2019efficient} refers to as ``LipSDP-neuron''; at \(\tau = N_f - 1\) we have the completely dense ``LipSDP-network''.
}
\label{fig:Ztaus}
\end{figure}

To compactly present our results, we first define a notation for summation as follows:
\begin{align*}
    S(k) \coloneqq \sum_{l = 1}^{k} n_l,
    \quad n_{K+1} \coloneqq m,
    \quad S(0) \coloneqq 0,
    \quad S(K) \coloneqq N.
\end{align*}
Our main results are then stated in the following theorems.

\begin{theorem}
\label{thm:Z-sparsity}
Let  \(Z(\gamma)\) be defined as in  \eqref{eq:Z_gamma}. It holds that    \(Z(\gamma) \in \mbb{S}^{N} (\mcal{E})\), where \(\mcal{E} \coloneqq \bigcup_{k = 1}^{K-1} \mcal{E}_k\) such that
    \begin{align*}
        \mcal{E}_k \coloneqq
        \big\{
            (i,j) :
            &\enskip S(k-1) + 1 \leq i, j \leq S(k+1) + \tau
        \big\}.
    \end{align*}
\end{theorem}

Note that the set $\mcal{E}$ defined in Theorem \ref{thm:Z-sparsity} is already illustrated in Fig. \ref{fig:Ztaus}.
Also, we implicitly assume that all \((i,j)\) in the definition of $\mcal{E}$ are within \(1 \leq i, j \leq N\).
From this construction of \(\mcal{E}\), it is then straightforward to prove chordality of \(\mcal{G}(\mcal{V}, \mcal{E})\) and identify its maximal cliques.

\begin{theorem}
\label{thm:max-cliques-Z}
    Let \(\mcal{V} \coloneqq \{1, \ldots, N\}\) and define \(\mcal{E}\) as in Theorem~\ref{thm:Z-sparsity}.
    Then \(\mcal{G}(\mcal{V}, \mcal{E})\) is chordal, and the set of its maximal cliques is \(\{\mcal{C}_1, \ldots, \mcal{C}_p\}\), where
    \begin{align*}
        p \coloneqq \min \braces{k : S(k+1) + \tau \geq N}
    \end{align*}
    and each clique \(\mcal{C}_k\) for \(k < p\) has size and elements
    \begin{align*}
        \abs{\mcal{C}_k} \coloneqq n_k + n_{k+1} + \tau,
        \quad
        \mcal{C}_k (i) \coloneqq S(k-1) + i
    \end{align*}
    for \(1 \leq i \leq \abs{\mcal{C}_k}\).
    The final clique \(\mcal{C}_p\) has elements
    \begin{align*}
        \mcal{C}_p (i) \coloneqq S(p - 1) + i,
        \quad 1 \leq i \leq N - S(p-1).
    \end{align*}
\end{theorem}

Using Lemma \ref{lem:chordal-psd}, the maximal cliques \(\{\mcal{C}_1, \ldots, \mcal{C}_p\}\) from Theorem \ref{thm:max-cliques-Z} now give a chordal decomposition of \(Z(\gamma)\), and lets us  formulate the following semidefinite program that we call Chordal-LipSDP:
\begin{align}
\label{eq:p2}
\begin{split}
    \underset{\gamma \geq 0, Z_1, \ldots, Z_p}{\text{minimize}} &\quad \gamma_{\ell} \\
    \text{subject to}
        &\quad Z (\gamma) = \sum_{k = 1}^{p} E_{\mcal{C}_k} ^\top Z_k E_{\mcal{C}_k}, \\
        &\quad Z_k \preceq 0 \enskip \text{for} \enskip k = 1, \ldots, p,
\end{split}
\end{align}
We remark that solving Chordal-LipSDP is typically much faster than solving LipSDP, especially for deep neural networks, as we impose a set of smaller semidefinite matrix constraints instead of one large semidefinite matrix constraint. In other words, the computational benefit of~\eqref{eq:p2} over~\eqref{eq:p1} is that each \(Z_k \preceq 0\) constraint is a significantly smaller LMI than \(Z(\gamma) \preceq 0\), which is especially the case for deeper networks.

In the next theorem, we show that LipSDP and Chordal-LipSDP compute Lipschitz constants that are, in fact, \emph{identical}, i.e., a chordal decomposition of LipSDP gives no loss of accuracy over the original formulation.
\begin{theorem}
\label{thm:equiv-probs}
    The semidefinite programs~\eqref{eq:p1} and~\eqref{eq:p2} are equivalent: \(\gamma\) is a solution for~\eqref{eq:p1} iff \(\gamma, Z_1, \ldots, Z_p\) is a solution for~\eqref{eq:p2}.
    Moreover, their optimal values are identical.
\end{theorem}

\section{Experiments}
\label{sec:experiments}

In this section, we evaluate the effectiveness of Chordal-LipSDP.
We aim to answer the following questions:
\begin{itemize}
    \item[] \textbf{(Q1)} How well does Chordal-LipSDP scale in comparison to the baseline methods?
    
    \item[] \textbf{(Q2)} How does the computed Lipschitz constant vary as the sparsity parameter \(\tau\) increases?
    
    
\end{itemize}

(\textbf{Dataset})
We use a randomly generated batch of neural networks with random weights from \(\mcal{N}(0, 1/2)\), with depth \(K = d\), widths \(n_2 = \cdots = n_K = w\), and input-output \(n_1 = m = 2\),
for \(w \in \{10, \ldots, 50\}\) and \(d \in \{5, 10, \ldots, 50\}\).
As a naming convention, for instance, W30-D20 would be the random network with \(w = 30\) and \(d = 20\).
In total, there are \(50\) such random networks.

(\textbf{Baseline Methods})
We compare Chordal-LipSDP against the following baselines: 
\begin{itemize}
    \item LipSDP: as in~\eqref{eq:p1}, using the same values of \(\tau\)
    \item Naive-Lip: by taking \(L = \prod_{k = 1}^{K} \norm{W_k}_2\)
    \item CP-Lip~\cite{combettes2020lipschitz}, which scales exponentially with depth.
    To the best of our knowledge, this is the only\footnote{The method of~\cite{virmaux2018lipschitz} is not a true upper-bound of the Lipschitz constant, although it is often such in practice as demonstrated in~\cite{fazlyab2019efficient}.
    The method of~\cite{jordan2020exactly} assumes piecewise linear activations.} other method that can handle general activation functions while yielding a non-trivial bound.    
\end{itemize}

(\textbf{System})
All experiments were run on an Intel i9-9940X with 28 cores and 125 GB of RAM.
We used MOSEK 9.3 as our convex solver with a solver accuracy \(\varepsilon = 10^{-6}\).

\subsection{(Q1) Runtime of Chordal-LipSDP vs Baselines}

We first evaluate the runtime of Chordal-LipSDP against the baselines of LipSDP, Naive-Lip, and CP-Lip.
For each random network, we ran both Chordal-LipSDP and LipSDP with sparsity parameter values of \(\tau = 0, \ldots, 6\) and recorded their respective runtimes in Figure~\ref{fig:facet-plot}.
Because Naive-Lip and CP-Lip do not depend on the sparsity parameter \(\tau\), they, therefore, appear as constant times for all sparsities; we omit plotting the Naive-Lip times because they are \(< 0.1\) seconds for all networks.
Moreover, because the runtime of CP-Lip scales exponentially with the number of layers, we only ran CP-Lip for networks of depth \(\leq 25\).

\begin{figure*}[h]
\centering

\includegraphics[width=\textwidth]{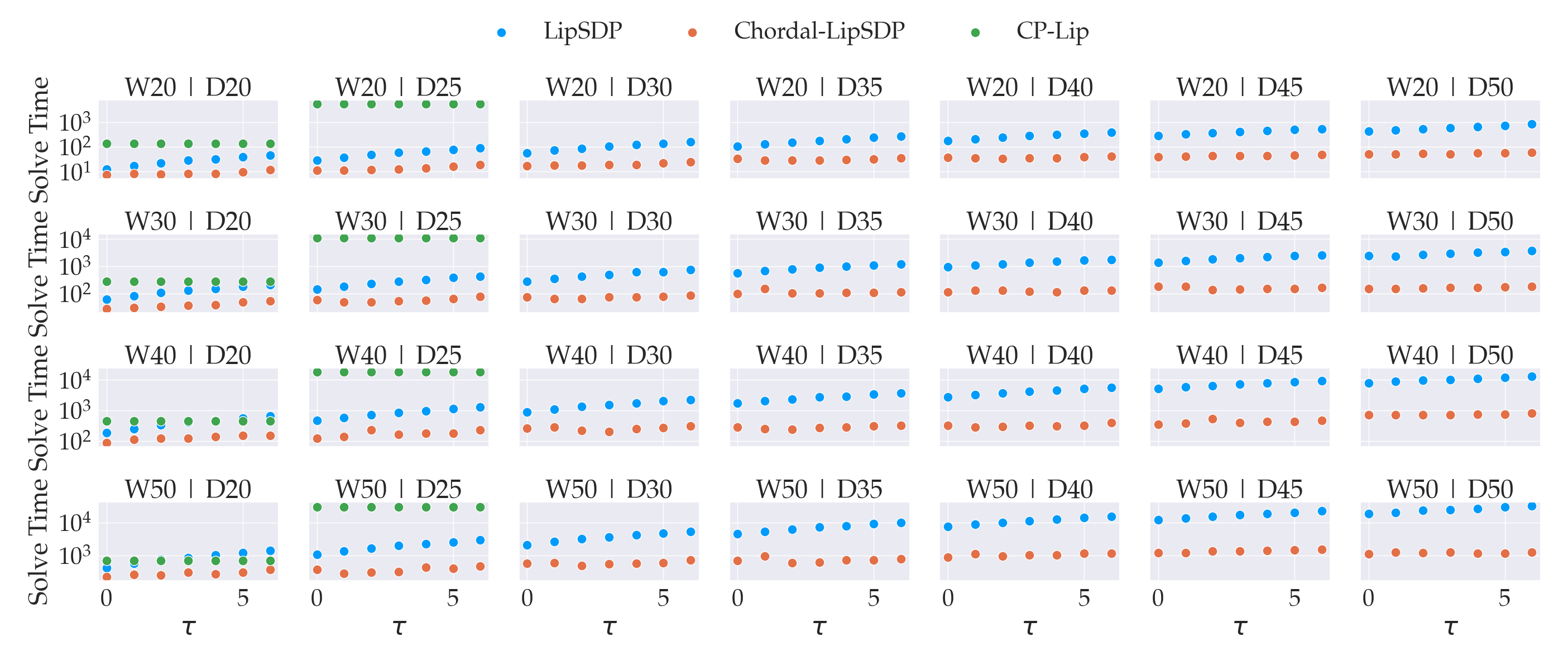}


\caption{
    The runtimes (seconds) of Chordal-LipSDP, LipSDP, and CP-Lip on a subset of the networks.
    The times for Naive-Lip are omitted because it finishes in \(< 0.1\) seconds on all instances.
    We ran Chordal-LipSDP and LipSDP for \(\tau = 0, \ldots, 6\).
    Because CP-Lip is independent of \(\tau\), it is a constant line.
    Moreover, due to the scaling exponentially with respect to the number of layers, we only ran CP-Lip for networks of depth \(\leq 25\).
}
\label{fig:facet-plot}
\end{figure*}

\begin{figure}[h]
\centering

\includegraphics[width=0.48\textwidth]{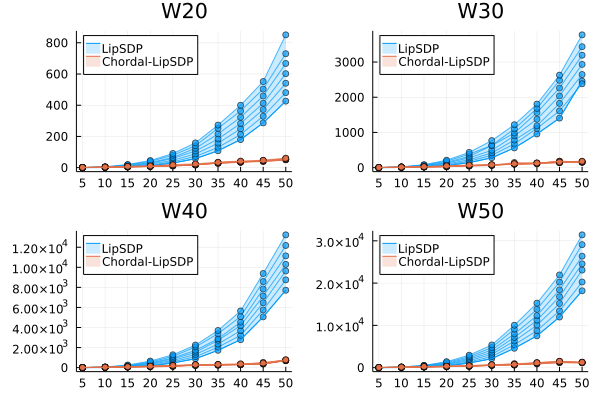}

\caption{
    The runtimes (seconds) of LipSDP and Chordal-LipSDP as the depth varies.
    Each plot shows networks that share the same width, but whose depths are varied on the x-axis.
    Each curve shows the runtimes for a different value of \(\tau = 0, \ldots, 6\), with higher curves corresponding to higher runtimes --- and in this case also higher values of \(\tau\).
    We shade the region between the \(\tau = 0\) and \(\tau = 6\) curves for each method.
}
\label{fig:fixed-widths}
\end{figure}

\begin{figure}[h]
\centering

\includegraphics[width=0.48\textwidth]{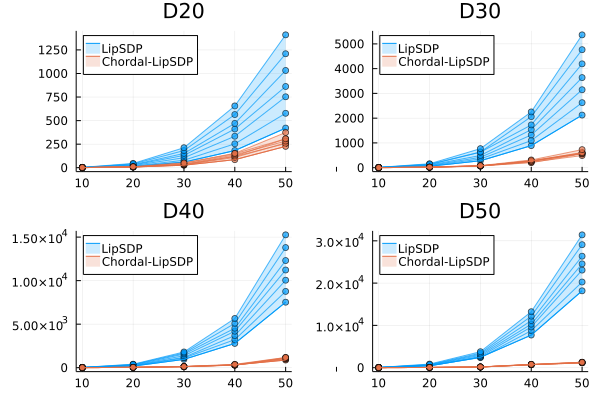}

\caption{
    The runtimes (seconds) of LipSDP and Chordal-LipSDP as the width varies.
    Similar to Figure~\ref{fig:fixed-widths}, but the x-axis now shows varying widths.
}
\label{fig:fixed-depths}
\end{figure}

Figure~\ref{fig:facet-plot} gives a general comparison for scalability between LipSDP and Chordal-LipSDP.
We further record the runtimes of these two methods when the width of the network is fixed, and the depths are varied in Figure~\ref{fig:fixed-widths}, as well as when the depth is fixed, and the widths may vary in Figure~\ref{fig:fixed-depths}.

As the network depth increases, Chordal-LipSDP significantly out-scales LipSDP, especially for networks of depth \(\geq 20\).
Moreover, Chordal-LipSDP also achieves better scaling for higher values of \(\tau\) compared to LipSDP.
In general, Naive-Lip is consistently the fastest method, while CP-Lip is initially fast but quickly falls off on deep networks due to exponential scaling with depth.

\subsection{(Q2) Lipschitz Constant vs Sparsity Parameter}

We also studied how the value of \(\tau\) affects the resulting Lipschitz constant and plot the results in Figure~\ref{fig:lc-taus}.
In particular, as \(\tau\) increases, the estimate rapidly improves by at least an order of magnitude.
Moreover, the Lipschitz constant estimate is also better than Naive-Lip and CP-Lip --- when the runtime would be reasonable (depth \(\leq 25\)).

\begin{figure}[H]
\centering

\includegraphics[width=0.48\textwidth]{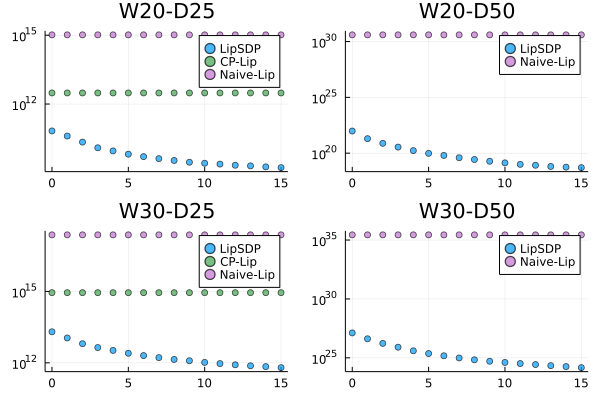}

\caption{
    The Lipschitz constant estimate given by LipSDP (the same as Chordal-LipSDP) on some networks, with \(\tau\) on the x-axis.
    On the left we also plot the estimates given by CP-Lip (green) and Naive-Lip (purple).
}
\label{fig:lc-taus}
\end{figure}

\subsection{Discussion}
Our experiments show that Chordal-LipSDP out-scales LipSDP on deeper networks, but this is not necessarily the case for shallower networks, e.g., when depth \(\leq 10\).
This is likely because the overhead of creating many smaller constraints of the form \(Z_k \preceq 0\), as well as a large equality constraint \(Z(\gamma) = \sum_{k = 1}^{p} E_{\mcal{C}_k} ^\top Z_k E_{\mcal{C}_k}\) may only be worthwhile when there are sufficiently many maximal cliques, i.e., when the network is deep.
This also means that Chordal-LipSDP is more resource-intensive than LipSDP: we found that on W50-D50, using \(\tau \geq 8\) will sometimes cause our machine to automatically terminate the process.

For LipSDP, we found that solving the dual problem was almost always significantly faster than solving the primal.
However, this dualization did not yield noticeable benefits for Chordal-LipSDP, so all instances of Chordal-LipSDP were solved for the primal.

Additionally, we found it helpful to scale the weights of \(W_k\) to ensure that the solver receives a sufficiently well-conditioned problem, especially for larger problem instances.

\section{Conclusions}

We present Chordal-LipSDP, a chordally sparse variant of LipSDP, for estimating the Lipschitz constant of a feedforward neural network.
We give a precise characterization of the sparsity structure present, and using this, we decompose a large semidefinite constraint of LipSDP --- which is its main computational bottleneck --- into an equivalent collection of smaller constraints.

Our numerical experiments show that Chordal-LipSDP significantly out-scales LipSDP, especially on deeper networks.
Moreover, our formulation introduces a tunable sparsity parameter that allows the user to finely trade off accuracy and scalability: it is often possible to gain rapidly tightening estimates of the Lipschitz constant without a major performance penalty.

\bibliographystyle{IEEEtran}
\bibliography{sources}

\appendix

\subsection{Proof of Theorem \ref{thm:Z-sparsity}}


To simplify and formalize the proof, we first need to introduce some useful notation. We extend the definition of sparsity patterns to general matrices.
Let \(\mcal{E} \subseteq \{1, \ldots, m\} \times \{1, \ldots, n\}\) and define analogously to~\eqref{eq:SE}:
\begin{align}
\label{eq:ME}
    \mbb{M}^{m \times n} (\mcal{E})
    \coloneqq
    \{M \in \mbb{R}^{m \times n} : 
        \text{\(M_{ij} = 0\) if \((i,j) \not\in \mcal{E}\)}\},
\end{align}
and for \((i,j) \in \mcal{E}\) associated with an \(m \times n\) matrix we will assume that \(1 \leq i \leq m\) and \(1 \leq j \leq n\).
When \(m = n\), we simply write \(\mbb{M}^{n}\).
Let \(\mcal{E}^\top\) be the transpositioned (inverse) pairs of \(\mcal{E}\), and note that \(\mcal{E} = \mcal{E}^\top\) iff \(\mcal{E}\) is symmetric.
We will explicitly distinguish between symmetric and nonsymmetric \(\mcal{E}\) when necessary. Whenever we write \(\mbb{S}^n (\mcal{E})\) it is implied that \(\mcal{E}\) is symmetric, and that the undirected graph \(\mcal{G}(\mcal{V}, \mcal{E})\) is therefore well-defined.

 In the remainder, we also use the following notation
\begin{align*}
    k_i \coloneqq \min \{k : S(k) \geq i\},
    \quad
    1 \leq i \leq N.
\end{align*}
There are a few useful properties for $k_i$ that we remark:
\begin{itemize}
    \item \(k_i\) is the index of \((n_1, \ldots, n_K)\) that \(1 \leq i \leq N\) falls in.
    \item If \(i \leq j\), then \(k_i \leq k_j\).
    \item \(S(k_i - 1) \leq i \leq S(k_i)\) for all \(1 \leq i \leq N\).
\end{itemize}
Also, some rules of sparse matrix arithmetics are as follows:
{\small\begin{align*}
    A \in \mbb{M}^{m \times n} (\mcal{E})
        &\implies A^\top \in \mbb{M}^{n \times m} (\mcal{E}^\top) \\
    A \in \mbb{M}^{n} (\mcal{E}_A), B \in \mbb{M}^{n} (\mcal{E}_B)
        &\implies A + B \in \mbb{M}^{n} (\mcal{E}_A \cup \mcal{E}_B) \\
    A \in \mbb{M}^{n} (\mcal{E})
        &\implies A + A^\top \in \mbb{S}^n (\mcal{E} \cup \mcal{E}^{\top})
\end{align*}}

To prove Theorem \ref{thm:Z-sparsity}, we need to show that \(Z(\gamma) \in \mbb{S}^N (\mcal{E})\).
Note first that \(Z(\gamma)\)  can be expressed as:
\begin{align*}
    Z(\gamma) &= A^\top T A + B^\top T B + A^\top T B + B^\top T A \\
    &\quad + E_K ^\top W_K ^\top W_K E_K - \gamma_l E_1 ^\top E_1.
\end{align*}
The proof of Theorem~\ref{thm:Z-sparsity} follows five steps and analyzes sparsity of each term in \(Z(\gamma)\) separately. For better readability, we summarize these five steps next and provide detailed proofs for each step in separate lemmas.

\textbf{Step 1.}  We  construct the edge set \(\mcal{E}_B \coloneqq \bigcup_{k = 1}^{K-1} \mcal{E}_{B,k}\) where
\begin{align*}
    \mcal{E}_{B, k}
        \coloneqq \big\{(i, j) :
        &\enskip S(k-1) + 1 \leq j \leq S(k), \\
        &\enskip S(k) - \tau + 1 \leq i \leq S(k+1) + \tau
        \big\}.
\end{align*} In Lemma \ref{lem:calB}, we show that \(B^\top T A \in \mbb{M}^N (\mcal{E}_B)\). By symmetry, it then also holds that \(A^\top T B \in \mbb{M}^{N} (\mcal{E}_B ^{\top})\).
    
\textbf{Step 2.} By construction, each \(\mcal{E}_{B, k}\) has dense entries only in the column range \(S(k-1) + 1 \leq j \leq S(k)\), which means that \(\mcal{E}_{B, k} \cap \mcal{E}_{B, k'} = \emptyset\) when \(k \neq k'\).
The goal now is to show that \(B^\top T A + A^\top T B\) is in a sense the ``frontier'' of growth for \(Z(\gamma)\) as \(\tau\) increases, as seen in Figure~\ref{fig:Ztaus}.
To more easily analyze the growth pattern of \(\mcal{E}_B \cup \mcal{E}_B^\top\), we define an over-approximation \( \mcal{E}_C \coloneqq \bigcup_{k = 1}^{K} \mcal{E}_{C, k} \supseteq\mcal{E}_{B}\) where
\begin{align*}
    \mcal{E}_{C, k}
        \coloneqq \big\{ (i,j) :
            &\enskip S(k-1) + 1 \leq j \leq S(k), \\
            &\enskip 1 \leq i \leq S(k+1) + \tau \big\}.
\end{align*}
Each \(\mcal{E}_{C, k}\) is similar to \(\mcal{E}_{B, k}\), but with the \(i\) index range relaxed.
In addition the union is up to \(K\), which is the depth of the neural network.
\(\mcal{E}_C\) is then a stair-case like sparsity pattern where the top side is dense (resp. the left side of \(\mcal{E}_C ^\top\) is dense), and so \(\mcal{E}_C \cap \mcal{E}_C ^\top\) is an overlapping block diagonal structure. Our goal is now to show that each term of \(Z(\gamma)\) has sparsity \(\mcal{E}_C \cap \mcal{E}_C ^\top\), beginning with \(B^\top T A + A^\top T B\). In Lemma~\ref{lem:calB-subset}, we show that \(\mcal{E}_B \cup \mcal{E}_B ^\top \subseteq \mcal{E}_C \cap \mcal{E}_C ^\top\). By Step 1, it consequently follows that 
 \begin{align*}B^\top T A + A^\top T B
            \in \mbb{S}^N (\mcal{E}_B \cup \mcal{E}_B^\top)
            \subseteq \mbb{S}^N (\mcal{E}_C \cap \mcal{E}_C ^\top).
    \end{align*}

\textbf{Step 3.} Let us next define the edge set \(\mcal{E}_A\) as
 \begin{align*}
    \mcal{E}_A \coloneqq \big\{(i, j) :
        &\enskip S(k_j) - \tau + 1 \leq S(k_i+1), \\
        &\enskip S(k_i) - \tau + 1 \leq S(k_j+1)
        \big\}.
    \end{align*}
    In Lemma~\ref{lem:calA}, we show that
    \begin{align*}
        A^\top T A + E_K ^\top W_K ^\top W_K E_K - \gamma_{\ell} E_1 ^\top E_1
        \in \mbb{S}^N (\mcal{E}_A),
    \end{align*}
    while we show that \(\mcal{E}_A \subseteq \mcal{E}_C \cap \mcal{E}_C ^\top\) in Lemma~\ref{lem:calA-subset}.
    
\textbf{Step 4.} For the  remaining term $B^\top T B$ of $Z(\gamma)$, we show that \(B^\top T B \in \mbb{S}^{N} (\mcal{E}_C \cap \mcal{E}_C ^\top)\) in Lemma~\ref{lem:BTB-subset}.
    
\textbf{Step 5.} The previous steps imply that \(Z(\gamma) \in \mbb{S}^N (\mcal{E}_C \cap \mcal{E}_C ^\top)\). Particularly, by Lemmas~\ref{lem:calB-subset}, \ref{lem:calA-subset}, and \ref{lem:BTB-subset}, each term has sparsity \(\mcal{E}_C \cap \mcal{E}_C ^\top\), and therefore so does their sum. Finally, we show that \(\mcal{E}_C \cap \mcal{E}_C ^\top \subseteq \mcal{E}\) in Lemma~\ref{lem:Eks}. This therefore means that \(Z(\gamma) \in \mbb{S}^N (\mcal{E})\) and concludes the proof. \hfill$\square$



\subsection{Statement and proof of Lemma \ref{lem:calB}}
\begin{lemma}
\label{lem:calB}
It holds that \(B^\top T A \in \mbb{M}^{N} (\mcal{E}_B)\).
\end{lemma}
\begin{proof}
    We analyze the action of \(B^\top T\) on each block column \(A_k \in \mbb{R}^{N_f \times n_k}\) of \(A\) separately, where,
    \begin{align*}
        A = \begin{bmatrix} A_1 & \cdots & A_{K-1} & 0 \end{bmatrix},
        \quad A = \sum_{k = 1}^{K-1} A_k E_k.
    \end{align*}
    Since \(W_k \in \mbb{R}^{n_{k+1} \times n_k}\), the entry \((A_k)_{ij}\) is dense iff
    \begin{align*}
        S(k) - n_1 + 1 \leq i \leq S(k+1) - n_1,
    \end{align*}
    and there is no condition on \(j\) because each column of \(A_k\) has at least one dense entry.
    Observe that \(T\) is a \(\tau\)-banded matrix, and therefore has the same sparsity as \(R + R^\top\), where
    \begin{align*}
        R \coloneqq I + U + \cdots + U^\tau,
        \quad \text{\(U\) is the upper shift matrix.}
    \end{align*}
    Thus \((TA_k)_{ij}\) is dense iff
    \begin{align*}
        S(k) - n_1 - \tau + 1 \leq i \leq S(k+1) - n_1 + \tau.
    \end{align*}
    Finally, left-multiplication by \(B^\top\) pads a zero block of height \(n_1\) at the top, and so \((B^\top T A_k)_{ij}\) is dense iff
    \begin{align*}
        S(k) - \tau + 1 \leq i \leq S(k+1) + \tau.
    \end{align*}
    Right-multiplication by \(E_k\) puts \(A_k\) into the \(k\)th block column of \((n_1, \ldots, n_K)\), so \((B^\top T A_k E_k)_{ij}\) is dense iff
    \begin{align*}
        S(k-1) + 1 &\leq j \leq S(k), \\
        \quad
        S(k) - \tau + 1 &\leq i \leq S(k+1) + \tau
    \end{align*}
    which shows that \(B^\top T A_k E_k\) has sparsity \(\mcal{E}_{B, k}\).
    Thus,    
    \begin{align*}
        B^\top T A = \sum_{k = 1}^{K-1} B^\top T A_k E_k
            \in \mbb{M}^N\parens{\bigcup_{k =1}^{K-1} \mcal{E}_{B, k}}
            = \mbb{M}^N (\mcal{E}_{B}).
    \end{align*}
\end{proof}

By symmetry we also have that \(A^\top T B \in \mbb{M}^N (\mcal{E}_B ^\top)\);
the dense blocks of \(B^\top T A\) grow vertically with \(\tau\), and those of \(A^\top T B\) grow horizontally.

   \subsection{Statement and proof of Lemma \ref{lem:calB-subset}}
    \begin{lemma}
\label{lem:calB-subset}
It holds that    \(\mcal{E}_B \cup \mcal{E}_B ^\top \subseteq \mcal{E}_C \cap \mcal{E}_C ^\top\).
\end{lemma}
\begin{proof}
    We show that \(\mcal{E}_{B, k} \subseteq \mcal{E}_C\) and \(\mcal{E}_{B, k} \subseteq \mcal{E}_C ^\top\) for any \(1 \leq k \leq K - 1\).
    It suffices to consider only \(\mcal{E}_{B, k}\) because \(\mcal{E}_C \cap \mcal{E}_C ^\top\) is symmetric, and would therefore also contain \(\mcal{E}_{B, k} ^\top\).
    
    To show that \(\mcal{E}_{B, k} \subseteq \mcal{E}_C\), observe that \(\mcal{E}_{B, k} \subseteq \mcal{E}_{C, k}\).
    To show that \(\mcal{E}_{B, k} \subseteq \mcal{E}_C^\top\), consider \((i,j) \in \mcal{E}_{B, k}\), and we claim that \((i,j) \in \mcal{E}_{C, k_i} ^\top\), for which we need to satisfy
    \begin{align*}
        S(k_i - 1) + 1 \leq i \leq S(k_i),
        \quad
        1 \leq j \leq S(k_i + 1) + \tau.
    \end{align*}
    The LHS inequalities follow from the previously stated properties of \(k_i\).
    For the RHS inequalities deduce the following from the definition of \(\mcal{E}_{B, k}\):
    \begin{align*}
        j \leq S(k), \enskip S(k) - \tau + 1 \leq i
        \implies j \leq S(k) \leq i + \tau - 1,
    \end{align*}
    and since \(i \leq S(k_i + 1)\) we have
    \begin{align*}
        j \leq i + \tau - 1 \leq S(k_i + 1) + \tau,
    \end{align*}
    meaning that \((i, j) \in \mcal{E}_{C, k_i} ^\top \subseteq \mcal{E}_C ^\top\).
\end{proof}

\subsection{Statement and proof of Lemma \ref{lem:calA}}
\begin{lemma}
\label{lem:calA}
It holds that    
\begin{align*}
A^\top T A + E_K ^\top W_K ^\top W_K E_K - \gamma_{\ell} E_1^\top E_1 \in \mbb{S}^N (\mcal{E}_A).
\end{align*}
   
\end{lemma}
\begin{proof}
    Note that \((\gamma_{\ell} E_1 ^\top E_1)_{ij}\) being dense implies that \((A^\top T A)_{ij}\) is dense, and therefore it suffices to show that
    \begin{align*}
        A^\top T A + E_K ^\top W_K ^\top W_K E_K
            &= W^\top \widehat{T} W
            \in \mbb{S}^N (\mcal{E}_A), \\
    \widehat{T} \coloneqq \mrm{blockdiag}(T, I),
        \quad
        W &\coloneqq \mrm{blockdiag}(W_1, \ldots, W_k),
    \end{align*}
    where it is assumed that each \(W_k \in \mbb{R}^{n_{k+1} \times n_k}\) is dense.
    Because \(\widehat{T}\) has more sparse entries than a \(\tau\)-banded matrix of the same size, it is therefore less dense than \(R^\top R\), where
    \begin{align*}
        R \coloneqq I + U + \cdots + U^\tau,
        \quad \text{\(U\) is the upper shift matrix}.
    \end{align*}
    Let \(V \coloneqq RW\), it then suffices to show that \(V^\top V \in \mbb{S}^N (\mcal{E}_A)\).
    Observe that \((V^\top V)_{ij} = \sum_l V_{li} V_{lj}\) is dense iff the \(i\)th and \(j\)th columns of \(V\) share a row \(\ell\) at which \(V_{li}\) and \(V_{lj}\) are both dense.
    Let \(V_k \in \mbb{R}^{(n_2 + \cdots + n_K + m) \times n_k}\) be the \(k\)th block column of \(V\), then \((V_k)_{ij}\) is dense iff
    \begin{align*}
        S(k) - n_1 - \tau + 1 \leq i \leq S(k + 1) - n_1.
    \end{align*}
    Thus \(V_k\) and \(V_{k'}\) have rows at which they are both dense iff
    \begin{align*}
        S(k) - n_1 - \tau + 1 &\leq S(k' + 1) - n_1 \\
        S(k') - n_1 - \tau + 1 &\leq S(k + 1) - n_1,
    \end{align*}
    which are equivalent to the conditions described in \(\mcal{E}_A\).
\end{proof}

\subsection{Statement and proof of Lemma \ref{lem:calA-subset}}
\begin{lemma}
\label{lem:calA-subset}
It holds that    \(\mcal{E}_A \subseteq \mcal{E}_C \cap \mcal{E}_C ^\top\).
\end{lemma}
\begin{proof}
    By symmetry of \(\mcal{E}_A\), it suffices to prove that \(\mcal{E}_A \subseteq \mcal{E}_C\).
    Consider \((i, j) \in \mcal{E}_A\), we claim that \((i,j) \in \mcal{E}_{C, k_j}\) --- for which a sufficient condition is
    \begin{align*}
        S(k_j - 1) + 1 \leq j \leq S(k_j),
        \quad
        1 \leq i \leq S(k_j + 1) + \tau,
    \end{align*}
    The LHS inequalities follow from the properties of $k_j$.
    For the RHS inequalities, recall that \(i \leq S(k_i)\), and rewrite the second condition of \(\mcal{E}_A\) to yield
    \begin{align*}
        1 \leq i \leq S(k_i) \leq S(k_j + 1) + \tau - 1,
    \end{align*}
    and so \((i, j) \in \mcal{E}_{C, k_j} \subseteq \mcal{E}_C\).
\end{proof}

\subsection{Statement and proof of Lemma \ref{lem:BTB-subset}}
\begin{lemma}
\label{lem:BTB-subset}
It holds that    \(B^\top T B \in \mbb{S}^N (\mcal{E}_C \cap \mcal{E}_C ^\top)\).
\end{lemma}
\begin{proof}
    By symmetry, it suffices to show \(B^\top T B \in \mbb{M}^{N} (\mcal{E}_C)\).
    Because left (resp. right) multiplication by \(B^\top\) (resp. \(B\)) consists of padding zeros on the top (resp. left), we may treat \(B^\top T B\) as a \(\tau\)-banded matrix.    
    First suppose that \(j \leq i\), then \((B^\top T B)_{ij}\) is dense iff \(i \leq j + \tau\).
    Since \(j \leq S(k_j+1)\),
    \begin{align*}
        1 \leq i \leq j + \tau
            \leq S(k_j + 1) + \tau,
    \end{align*}
    which shows that \((i,j) \in \mcal{E}_{C, k_j} \subseteq \mcal{E}_C\).
    
    Now suppose that \(i \leq j\), then \((B^\top T B)_{ij}\) is dense iff \(j \leq i + \tau\).
    Furthermore, \(S(k_i) \leq S(k_j) \leq S(k_j+1)\), so
    \begin{align*}
        1 \leq j \leq i + \tau \leq S(k_i) + \tau \leq S(k_j + 1) + \tau
    \end{align*}
    which again shows that \((i, j) \in \mcal{E}_{C, k_j} \subseteq \mcal{E}_C\).
\end{proof}

\subsection{Statement and proof of Lemma \ref{lem:Eks}}
\begin{lemma}
\label{lem:Eks}
It holds that    \(\mcal{E}_C \cap \mcal{E}_C ^\top \subseteq \mcal{E}\).
\end{lemma}
\begin{proof}
    Consider \((i,j) \in \mcal{E}_C \cap \mcal{E}_C ^\top\) and suppose without loss of generality that \(i \leq j\).
    Then \((i,j) \in \mcal{E}_{C, k_j}\) and \((i,j) \in \mcal{E}_{C, k_i}^\top\), meaning that the following conditions hold:
    \begin{align*}
        S(k_j - 1) + 1 \leq j \leq S(k_j), &\quad 1 \leq i \leq S(k_j +1) + \tau, \\
        S(k_i - 1) + 1 \leq i \leq S(k_i), &\quad 1 \leq j \leq S(k_i + 1) + \tau.
    \end{align*}
    Tightening the bounds on \(i\) and by monotonicity of \(S\),
    \begin{align*}
        S(k_i - 1) + 1 \leq i \leq S(k_i) \leq S(k_i + 1) + \tau,
    \end{align*}
    and also for \(j\) we have
    \begin{align*}
        S(k_i -1) + 1 \leq S(k_j - 1) + 1 \leq j \leq S(k_i + 1) + \tau,
    \end{align*}
    which together imply that \((i,j) \in \mcal{E}_{k_i} \subseteq \mcal{E}\).
\end{proof}

\subsection{Proof of Theorem \ref{thm:max-cliques-Z}}
The structure of \(\mcal{E}\) results in a lengthy proof of Theorem~\ref{thm:Z-sparsity}.
However, it is easy to guess each \(\mcal{E}_k\) by simple experiments, and leads to a straightforward proof of Theorem~\ref{thm:max-cliques-Z}.

    Observe that \(\mbb{S}^{N} (\mcal{E}_k)\) is the set of block diagonal matrices whose \((i,j)\) entry is dense iff
    \begin{align*}
        S(k-1) + 1 \leq i, j \leq S(k+1) + \tau.
    \end{align*}
    Because \(\mcal{E}\) is a union of the \(\mcal{E}_k\) sparsities, \(\mbb{S}^N (\mcal{E})\) is therefore the set of matrices with overlapping block diagonals, which are known to be chordal~\cite[Section 8.2]{vandenberghe2015chordal}.

    It remains to identify the maximal cliques of \(\mcal{G}(\mcal{V}, \mcal{E})\).
    First consider \(k < p\) with \(k \geq 1\), and observe that \((N, N) \not\in \mcal{E}_k\).
    By construction each \(\mcal{E}_k\) is the edges of a clique, and when \(k < p\)  such \(\mcal{E}_k\) is also not contained by any other \(\mcal{E}_{k'}\) because
    \begin{align*}
        \begin{dcases}
        S(k - 1) + 1 < S(k' - 1) + 1, &\enskip \text{if \(k < k'\)} \\
        S(k' + 1) + \tau < S(k + 1) + \tau, &\enskip \text{if \(k > k'\)},
        \end{dcases}
    \end{align*}
    meaning that there exists \((i,j) \in \mcal{E}_k \setminus \mcal{E}_{k'}\) with
    \begin{align*}
        \begin{dcases}
        i = j = S(k-1) + 1 &\enskip \text{if \(k < k'\)} \\
        i = j = S(k+1) + \tau &\enskip \text{if \(k > k'\)}.
        \end{dcases}
    \end{align*}
    Consequently, \(\mcal{E}_k\) is in fact the edges of a maximal clique containing indices \(i\) that satisfy
    \begin{align*}
        S(k-1) + 1 \leq i \leq S(k+1) + \tau,
    \end{align*}
    which are exactly the conditions of \(\mcal{C}_k\) for \(k < p\).
    
    Now consider \(k > p\) with \(k \leq K - 1\), and observe that \(\mcal{E}_k \subseteq \mcal{E}_p\) because any \((i,j) \in \mcal{E}_k\) will satisfy
    \begin{align*}
        S(p-1) + 1 < S(k -1) + 1 \leq i, j \leq N \leq S(p+1) + \tau,
    \end{align*}
    and so \((i,j) \in \mcal{E}_p\) as well.
    \(\mcal{E}_p\) is therefore the edges of a clique that contains all other cliques that contain \(N\), and is thus maximal --- corresponding to the description of \(\mcal{C}_p\). \hfill$\square$

\subsection{Proof of Theorem \ref{thm:equiv-probs}}
    Because \(Z(\gamma) \in \mbb{S}^{N} (\mcal{E})\) and \(\mcal{G}(\mcal{V}, \mcal{E})\) is chordal with maximal cliques \(\{\mcal{C}_1, \ldots, \mcal{C}_p\}\), conclude from Lemma~\ref{lem:chordal-psd} that \(Z(\gamma) \preceq 0\) iff each \(Z_k \preceq 0\). \hfill$\square$

\end{document}